\documentclass[manuscript]{stjour}
\journalname{Open Mind}

\supplementslinks{}

\received{}
\accepted{}
\published{}

\setdoi{}

\usepackage{float}
\usepackage[english]{babel}
\usepackage[utf8x]{inputenc}
\usepackage{amsfonts,amsmath,amssymb, amsthm}
\usepackage{stmaryrd}
\usepackage{graphicx}
\usepackage{multirow}
\usepackage{footmisc}
\usepackage{caption,setspace}
\usepackage{hyperref}
\newtheorem{theorem}{Theorem}
\nolinenumbers

\begin{document}

\title{A pragmatic account of the weak evidence effect}

\author[Samuel A. Barnett, Thomas L. Griffiths, Robert D. Hawkins]{Samuel A. Barnett\affil{1}, Thomas L. Griffiths\affil{1,2} \and Robert D. Hawkins\affil{2}\thanks{Materials and code for reproducing all behavioral experiments, analyses, and model comparisons are available online at \texttt{\href{https://github.com/s-a-barnett/bayesian-persuasion.git}{github.com/s-a-barnett/bayesian-persuasion}}.
Our study was pre-registered at \texttt{\href{ https://osf.io/gpbzu}{ https://osf.io/gpbzu}}}}

\affiliation{1}{Department of Computer Science, Princeton University, Princeton, New Jersey}
\affiliation{2}{Department of Psychology, Princeton University, Princeton, New Jersey}

\correspondingauthor{Robert D. Hawkins}{rdhawkins@princeton.edu}


\begin{abstract}
Language is not only used to transmit neutral information; we often seek to \emph{persuade} by arguing in favor of a particular view.
Persuasion raises a number of challenges for classical accounts of belief updating, as information cannot be taken at face value. 
How should listeners account for a speaker's ``hidden agenda'' when incorporating new information?
Here, we extend recent probabilistic models of recursive social reasoning to allow for persuasive goals and show that our model provides a \emph{pragmatic} account for why weakly favorable arguments may backfire, a phenomenon known as the weak evidence effect.
Critically, this model predicts a systematic relationship between belief updates and expectations about the information source: weak evidence should only backfire when speakers are expected to act under persuasive goals and prefer the strongest evidence.
We introduce a simple experimental paradigm called the \emph{Stick Contest} to measure the extent to which the weak evidence effect depends on speaker expectations, and show that a pragmatic listener model accounts for the empirical data better than alternative models.
Our findings suggest further avenues for rational models of social reasoning to illuminate classical decision-making phenomena.
  
\end{abstract}
\newpage

\begin{quote}
``Well, he would [say that], wouldn't he?'' \\ 
--- \textit{Mandy Rice-Davies, 1963}
\end{quote}

\section{Introduction}

Communication is a powerful engine of learning, enabling us to efficiently transmit complex information that would be costly to acquire on our own \citep{tomasello2009cultural, henrich2015secret}.
While much of what we know is learned from others, it can also be challenging to know how to incorporate socially transmitted information into our beliefs about the world. 
Each source is a person with a ``hidden agenda'' encompassing their own beliefs and desires and biases, and not all information can be treated the same \citep{hovland1953communication,o2015persuasion}. 
For example, when deciding whether to buy a car, we may weight information differently depending on whether we heard it from a trusted family memory or the dealership, as we know the dealership is trying to make a sale.
While such reasoning is empirically well-established --- even young children are able to discount information from untrustworthy or unknowledgeable individuals \citep{wood2013whom,sobel2013knowledge,poulin2016developmental,mills2016learning,gweon2014sins,harris2018cognitive} --- these phenomena have continued to pose a problem for formal models of belief updating, which typically take information at face value. 

Recent probabilistic models of social reasoning have provided a mathematical framework for understanding how listeners ought to draw inferences from socially transmitted information. 
Rather than treating information as a direct observation of the true state of the world, social reasoning models suggest treating the true state of the world as a \emph{latent variable} that can be recovered by inverting a generative model of how an intentional agent would share information under different circumstances 
\citep{hawthorne2019reasoning,velez2019integrating,WhalenEtAl18_SensitivityToSharedInfo,jara2016naive,baker2017rational,goodman_pragmatic_2016,goodman2013knowledge}.
These models raise new explanations for classic effects in the judgment and decision-making literature, where behavior is often measured in social or linguistic contexts \citep{politzer2000reasoning,sperber1995relevance,mosconi2001role,bagassi2006pragmatic, mckenzie2003speaker,ma2020delay}. 

Consider the \emph{weak evidence effect}~\citep{mckenzie_when_2002,fernbach_when_2011,lopes1987procedural} or \emph{boomerang effect}~\citep{petty2018attitudes}, a striking case of non-monotonic belief updating where weak evidence in favor of a particular conclusion may backfire and actually reduce an individual's belief in that conclusion. 
For example, suppose a juror is determining the guilt of a defendant in court.
%In the absence of any information about the crime, they may be agnostic about the outcome, assigning equal possibility to guilt and innocence. 
After hearing a prosecutor give a weak argument in support of a guilty verdict -- say, calling a single witness with circumstantial evidence -- we might expect the juror's beliefs to only be shifted weakly in support of guilt. 
Instead, the weak evidence effect describes a situation where the prosecutor's argument actually leads to a shift in the opposite direction -- the juror may now believe that the defendant is more likely to be \emph{innocent}.

Importantly, social reasoning mechanisms are not necessarily in conflict with previously proposed mechanisms for the weak evidence effect, such as algorithmic biases in generating alternative hypotheses \citep{fernbach_when_2011, dasgupta2017hypotheses}, causal reasoning about other non-social attributes of the situation \citep{bhui2020paradoxical} or sequential belief-updating \citep{mckenzie_when_2002,trueblood2011quantum}.
Both social and asocial models are able to account for the basic effect. 
To find \emph{unique} predictions that distinguish models with a social component, then, we argue that we must shift focus from the \emph{existence} of the effect to asking \emph{under what conditions} it emerges. 
Social mechanisms lead to unique predictions about these conditions that purely asocial models cannot generate.
In particular, if evidence comes from an intentional agent who is expected to present the strongest possible argument in favor of their case, then weak evidence would imply the absence of stronger evidence \citep{grice_logic_1975}; otherwise weak evidence may be taken more at face value.
Thus, a pragmatic account predicts a systematic relationship between a listener's social expectations and the strength of the weak evidence effect:\footnote{\citet{harris_james_2013} presents a related model of the \emph{faint praise} effect, where the omission of any stronger information that a speaker would be expected to know implies that it is more likely to be negative than positive (e.g.~``James has very good handwriting.'') Importantly, this effect is sensitive to the perceived expertise of the source; no such implication follows for unknowledgable informants \citep[see also][for related inferences from omission]{hsu2017absence,bonawitz2011double,gweon2014sins}.} \emph{weak evidence should only backfire when the information source is expected to provide the strongest evidence available to them.}

In this paper, we proceed by first extending recent rational models of communication to equip speakers with persuasive goals (rather than purely informative ones) and present a series of simulations deriving key predictions from our model. 
We then introduce a simple behavioral paradigm, the \emph{Stick Contest}, which allows us to elicit a participant's social expectations about the speaker alongside their inferences as listeners. 
Based on speaker expectation data, we find that participants cluster into sub-populations of \emph{pragmatic} listeners or \emph{literal} listeners, who expect speakers to provide strongly persuasive evidence or informative but neutral evidence, respectively.
As predicted by the pragmatic account, only the first group of participants, who expected speakers to provide persuasive evidence, reliably displayed a weak evidence effect in their belief updates. 
Finally, we use these data to quantitatively compare our model against prior asocial accounts and find that a pragmatic model accounting for these hetereogenous groups is most consistent with the empirical data.
Taken together, we suggest that pragmatic reasoning mechanisms are central to explaining belief updating when evidence is presented in social contexts.

\section{Formalizing a pragmatic account of the weak evidence effect}

To derive precise behavioral predictions, we begin by formalizing the pragmatics of persuasion in a computational model.
Specifically, we draw upon recent progress in the Rational Speech Act (RSA) framework \citep{franke2016probabilistic,goodman_pragmatic_2016,scontras2017probabilistic}. 
This framework instantiates a theory of recursive social inference, whereby listeners do not naively update their beliefs to reflect the information they hear, but explicitly account for the fact that speakers are intentional agents choosing which information to provide \citep{grice_logic_1975}.

\subsection{Reasoning about evidence from informative speakers}

We begin by defining a pragmatic listener $L$ who is attempting to update their beliefs about the underlying state of the world $w$ (e.g. the guilt or innocence of the defendant), after hearing an utterance $u$ (e.g. an argument provided by the prosecution). 
According to Bayes' rule, the listener's posterior beliefs about the world $P_L(w \mid u)$ may be derived as follows:
\begin{equation}
\label{eq:vanilla_listener}
    P_L (w \mid u) \propto P_S (u \mid w) P(w)
\end{equation}
where $P(w)$ is the listener's prior beliefs about the world and the likelihood $P_S(u \mid w)$ is derived by imagining what a hypothetical speaker agent would choose to say in different circumstances. 
This term yields different predictions given different assumptions about the speaker, captured by different speaker utility functions $U$. 
In existing RSA models, the speaker is usually assumed to be \emph{epistemically informative}, choosing utterances that bring the listener's beliefs as close as possible to the true state of the world, as measured by information-theoretic surprisal:
\begin{align}
    P_S (u \mid w) & \propto \exp\{\alpha U_{\textrm{epi}}(u; w)\}\nonumber \\
    U_{\textrm{epi}}(u;w) & = \ln P_{L_0} (w \mid u) \label{eq:epiU}
\end{align}
where the free parameter $\alpha\in[0,\infty]$ controls the temperature of the soft-max function and $U_{\textrm{epi}}$ denotes the utility function of an (epistemically) informative speaker. 
As $\alpha \rightarrow \infty$, the speaker increasingly chooses the single utterance with the highest utility, and as ${\alpha\rightarrow 0}$ the speaker becomes indifferent among utterances.
If this hypothetical speaker, in turn, aimed to be informative to the same listener defined in Eq.~\ref{eq:vanilla_listener}, it would yield an infinite recursion: the RSA framework instead assumes that the recursion is grounded in a base case known as the ``literal'' listener, $L_0$, who takes evidence at face value:
\begin{equation}
    P_{L_0} (w \mid u) \propto \delta_{\llbracket u \rrbracket (w)} P(w)\label{eq:L0}.
\end{equation}
Here, $\llbracket u \rrbracket$ gives the literal semantics of the utterance $u$, with $\delta_{\llbracket u \rrbracket (w)}$ returning 1 if $w$ is consistent with the state of affairs denoted by $u$, and 0 (or very small $\epsilon$) otherwise.

\subsection{Reasoning about evidence from motivated speakers}

The epistemic utility defined in Eq.~\ref{eq:epiU} aims only to produce assertions that most effectively lead to \emph{true} beliefs.
Often, however, speakers do not seek to neutrally inform, but to persuade in favor of a particular outcome or ``hidden agenda.''
What is needed to represent such persuasive goals in the RSA framework?
We begin by assuming that motivated speakers have a particular goal state  $w^*$ that they aim to induce in the listener, where $w^*$ does not necessarily coincide with the true state of affairs $w$.
This naturally yields a persuasive utility $U_{\textrm{pers}}$ that aims to persuade the listener to adopt the intended beliefs $w^*$:
\begin{align}
U_{\textrm{pers}}(u; w^*) & = \ln P_{L_0}(w^* \mid u) \label{eq:persU}
\end{align}
where we say an utterance $u$ is strictly more persuasive than $u'$ if and only if $U_{\textrm{pers}}(u \mid  w^*) > U_{\textrm{pers}}(u' \mid  w^*)$ (i.e. when the utterance results in the listener assigning higher probability to the desired state $w^*$).
Following prior extensions of the speaker utility to other non-epistemic goals \cite[e.g.][]{yoon2018balancing,yoon2020polite,bohn2021young}, we then define a combined utility assuming the speaker aims to jointly fulfill persuasive aims (Eq.~\ref{eq:persU}) while remaining consistent with the true world state $w$ (Eq.~\ref{eq:epiU}):
\begin{align}
    P_S (u \mid w, w^*) & \propto \exp\{\alpha \cdot U(u; w, w^*)\} \label{eq:s1}\\
    U(u; w, w^*) & = U_{\textrm{epi}}(u; w) + \beta\cdot U_{\textrm{pers}}(u; w^*)\label{eq:combU}
\end{align}
where $\beta$ is a parameter controlling the strength of the persuasive goal (we recover the standard epistemic RSA model when $\beta = 0$).
This motivated speaker forms the foundation for a pragmatic model of the weak evidence effect.\footnote{Coincident with our work, \cite{vignero2022updating} has proposed a similar formulation to explain how speakers may stretch the truth of epistemic modals like ``possibly'' or ``probably.''}
A pragmatic listener $L_1$ who suspects that the utterance was generated by a motivated speaker with non-zero bias $\beta$ is able to be ``skeptical'' of the speaker's agenda and discount their evidence accordingly:\footnote{Although we formulate the listener's posterior as being conditioned on a \emph{known} value of $\beta$, we can also consider the case in which the listener has a prior distribution over biases and can compute (marginal) posteriors accordingly -- refer to Appendix E for details.}
\begin{equation}\label{eq:L1}
    P_{L} (w \mid u, w^*, \beta) \propto P_{S} (u \mid w^*, w, \beta) \cdot P(w) 
\end{equation}
To see why this model allows evidence to backfire, note that the probability of different utterances are in competition with one another under the speaker model.
In the case that $w$ and $w^*$ coincide, the speaker is expected to choose a utterance that is strongly supportive of that state; weaker utterances have a lower probability of being chosen. 
Conversely, if $w^*$ deviates from the true state of affairs, stronger utterances in favor of $w^*$ will be dispreferred (because they will be false and violate the epistemic term), hence weaker utterances are more likely. 
In this way, the absence of strong evidence from a speaker who would be highly motivated to show it statistically implies that no such evidence exists. 

\begin{figure}[t]
    \captionsetup{font={small,stretch=.75}, strut=on}
    \centering
    \includegraphics[width=0.88\linewidth]{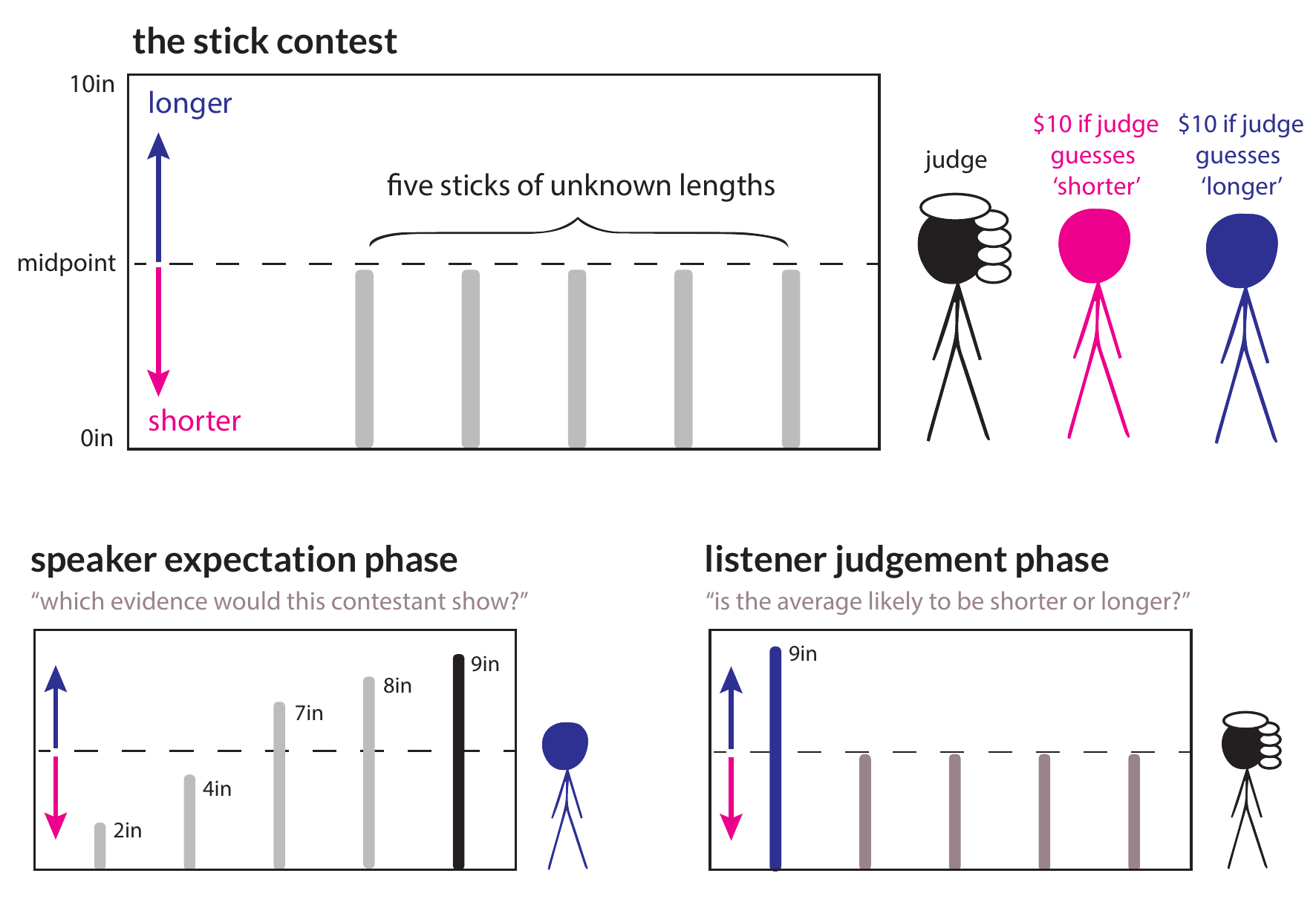}
  
    \caption{ In the Stick Contest paradigm, participants are asked to determine whether a set of five hidden sticks is longer or shorter, on average, than a midpoint (dotted line) based on limited evidence from a pair of contestants. In the \emph{speaker expectation} phase (left), participants were asked which one of the five sticks a given contestant would be most likely to show. In the \emph{listener judgment} phase (right), participants were presented with a sequence of sticks from each contestant and asked to judge the likelihood that the overall sample is ``longer.''}
    \label{fig:stickStimulus}
\end{figure}

\section{Experiment: The Stick Contest}

Empirical studies of the weak evidence effect require a cover story to elicit belief judgments and manipulate the strength of evidence.
Typically, this cover story is based on a real-world scenario such as a jury trial \citep{mckenzie_when_2002} or public policy debate \citep{fernbach_when_2011}, where participants are asked to report their belief in a hypothetical state such as the defendant's guilt or the effectiveness of the policy intervention. 
While these cover stories are naturalistic, they also introduce several complications for evaluating models of belief updating: participants may bring in different baseline expectations based on world knowledge and the absolute scalar argument strength of verbal statements is often unclear.
%  (e.g. different base rates of guilt or policy effectiveness)
%  (e.g. what is the strongest or weakest conceivable argument a politician could give for a policy proposal and how much better or worse would it be than the given argument).
To address these concerns, we introduce a simple behavioral paradigm called \emph{the Stick Contest} (see Fig.~\ref{fig:stickStimulus}).
This game is inspired by a courtroom scenario: two contestants take turns presenting competing evidence to a judge, who must ultimately issue a verdict.
Here, however, the verdict concerns the average length of $N=5$ sticks which range from a minimum length of 1'' to a maximum length of 9''.
These sticks are hidden from the judge but visible to both contestants, who are each given an opportunity to reveal exactly one stick as evidence for their case.
As in a courtroom, each contestant has a clear agenda that is known to the judge: one contestant is rewarded if the judge determines that the average length of the sticks is longer than the midpoint of 5'' (shown as a dotted line in Fig.~\ref{fig:stickStimulus}), and the other is rewarded if the judge determines that the average length of the sticks is shorter than the midpoint.

This paradigm has several advantages for comparing models of the weak evidence effect.
First, unlike verbal statements of evidence, the scale of evidence strength is made explicit and provided as common knowledge to the judge and contestants. 
The strength of a given piece of evidence is directly proportional to the length of the revealed stick, and these lengths are bounded between the minimum and maximum values. 
Second, while previous paradigms have operationalized the weak evidence effect in terms of a sequence of belief updates across multiple pieces of evidence (e.g. where the first piece of evidence sets a baseline for the second piece of evidence), common knowledge about the scale allows the weak evidence effect to emerge from a single piece of evidence.
This property helps to disentangle the core mechanisms driving the weak evidence effect from those driving \emph{order effects} \cite[e.g.][]{trueblood2011quantum}.

\subsection{Participants}
We recruited $804$ participants from the Prolific crowd-sourcing platform,  $723$ of whom successfully completed the task and passed attention checks (see Appendix A).
The task took approximately 5 to 7 minutes, and each participant was paid \$1.40 for an average hourly rate of \$14.
We restricted recruitment to the USA, UK, and Canada and balanced recruitment evenly between male and female participants. 
Participants were not allowed to complete the task on mobile or to complete the experiment more than once.

\subsection{Design and procedure}

The experiment proceeded in two phases: first, a \emph{speaker expectation} phase, and second, a \emph{listener judgment} phase (see~Fig.~\ref{fig:stickStimulus}). 
In the speaker expectation phase, we placed participants in the role of the contestants, gave them an example set of sticks $\{2,4,7,8,9\}$ and asked them which ones they believed each contestant would choose to show, in order of priority.
In the listener judgment phase, we placed participants in the role of the judge and presented them with a sequence of observations. 
After each observation, they used a slider to indicate their belief about the verdict on a scale ranging from 0 (``average is definitely shorter than five inches'') to 100 (``average is definitely longer than five inches'').
It was stated explicitly that the judge knows that there are exactly five sticks, and that each contestant's incentives are public knowledge. 
After each phase, we asked participants to explain their response in a free-response box (see Tables S2-S3 for sample responses).

This within-participant design allowed us to examine individual co-variation between the strength of a participant's weak evidence effect in the listener judgment phase and their beliefs about the evidence generation process in the speaker expectation phase.  
Critically, while the set of candidate sticks in the speaker expectation phase was held constant across all participants for consistency, the strength of evidence we presented in the listener judgment phase was manipulated in a between-subjects design.
The length of the first piece of evidence was chosen from the set $\{6, 7, 8, 9\}$ when the long-biased contestant went first, and from the set $\{4, 3, 2, 1\}$ when the short-biased contestant went first, for a total of 4 possible ``strength'' conditions (measured as the distance of the observation from the midpoint; we assigned more participants to the more theoretically important ``weak evidence'' condition, i.e. $\{4, 6\}$, to obtain a higher-powered estimate).
The order of contestants was counterbalanced across participants and held constant across the speaker and listener phase.\footnote{An earlier iteration of our experiment only used a \texttt{long}-biased speaker; we report results from this version in Appendix D.}
Although it was not the focus of the current study, we also presented a second piece of evidence from the other contestant to capture potential order effects (see Appendix B for preliminary analyses).

\section{Results}

\subsection{Behavioral results}

\begin{figure}[t]
    \centering
        \captionsetup{font={small,stretch=.75}, strut=on}

    \includegraphics[width=0.7\linewidth]{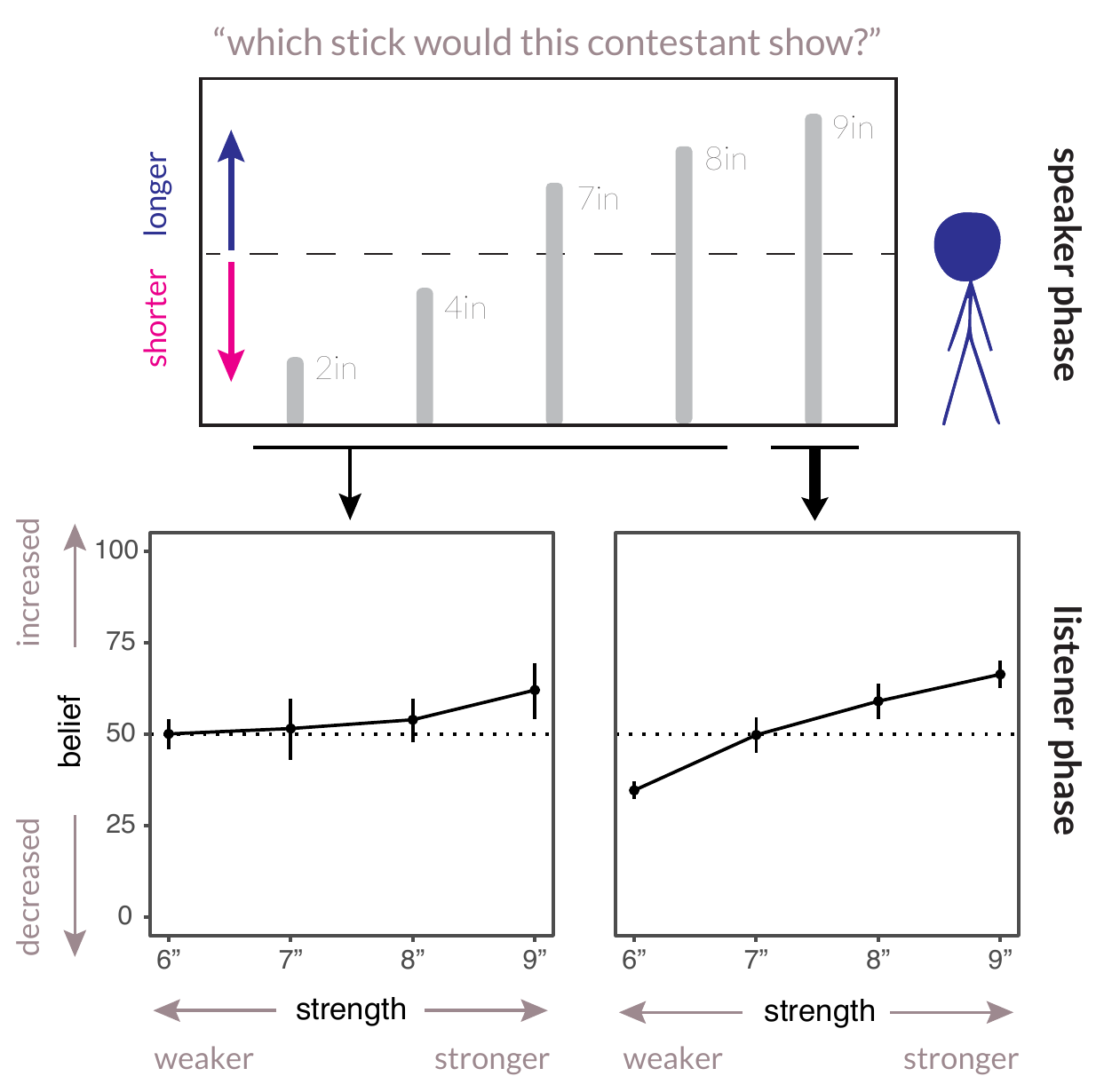}
    \caption{Individual differences in the weak evidence effect are predicted by pragmatic expectations. Dotted line represents neutral or unchanged beliefs. Error bars are bootstrapped 95\% CIs (see Fig.~S3 for raw distributions).}
    \label{fig:behavior}
\end{figure}

Before quantitatively evaluating our model, we first examine its key qualitative predictions. 
Do participants exhibit a weak evidence effect in their listener judgments at all, and if so, to what extent is variation in the strength of the effect related to their expectations about the speaker? 
We focus on each participant's first judgment, provided after the first piece of evidence in the listener phase. 
This judgment provides the clearest view of the weak evidence effect, as subsequent judgments may be complicated by order effects.
We constructed a linear regression model predicting participants' continuous slider responses. 
We included fixed effects of evidence strength as well as expectations from the speaker phase (coded as a categorical variable, expecting strongest evidence vs. expecting weaker evidence), and their interaction, along with a fixed effect of whether the first contestant was ``short''-biased or ``long''-biased. 
Because the design was fully between-participant (i.e. each participant only provided a single slider response as judge), no random effects were supported.

As predicted, we found a significant interaction between speaker expectations and evidence strength, $t(718) = 5.2,~p < 0.001;$ see Fig.~\ref{fig:behavior}.
For participants who expected the speaker to provide the strongest evidence (485 participants or 67\% of the sample), weak evidence in favor of the persuasive goal backfired and actually pushed beliefs in the opposite direction, $m = 34.7$,~95\% CI: $[32.3, 37.3], p < 0.001$. 
Meanwhile, for participants who expected speakers to ``hedge'' and not necessarily show the strongest evidence first (238 participants, or 33\% of the sample), no weak evidence effect was found ($m = 50.1$, group difference $= -15.4$,~post-hoc $t(367) = -6.3,~p<0.001$.)
We found only a marginally significant asymmetry in slider bias, $p=0.056$, with short-biased participants giving slightly larger endorsements ($m=1.6$ slider points) across the board.

\subsection{Model simulations}

\begin{figure}[t]
    \captionsetup{font={small,stretch=.75}, strut=on}

    \centering
    \includegraphics[scale=0.35]{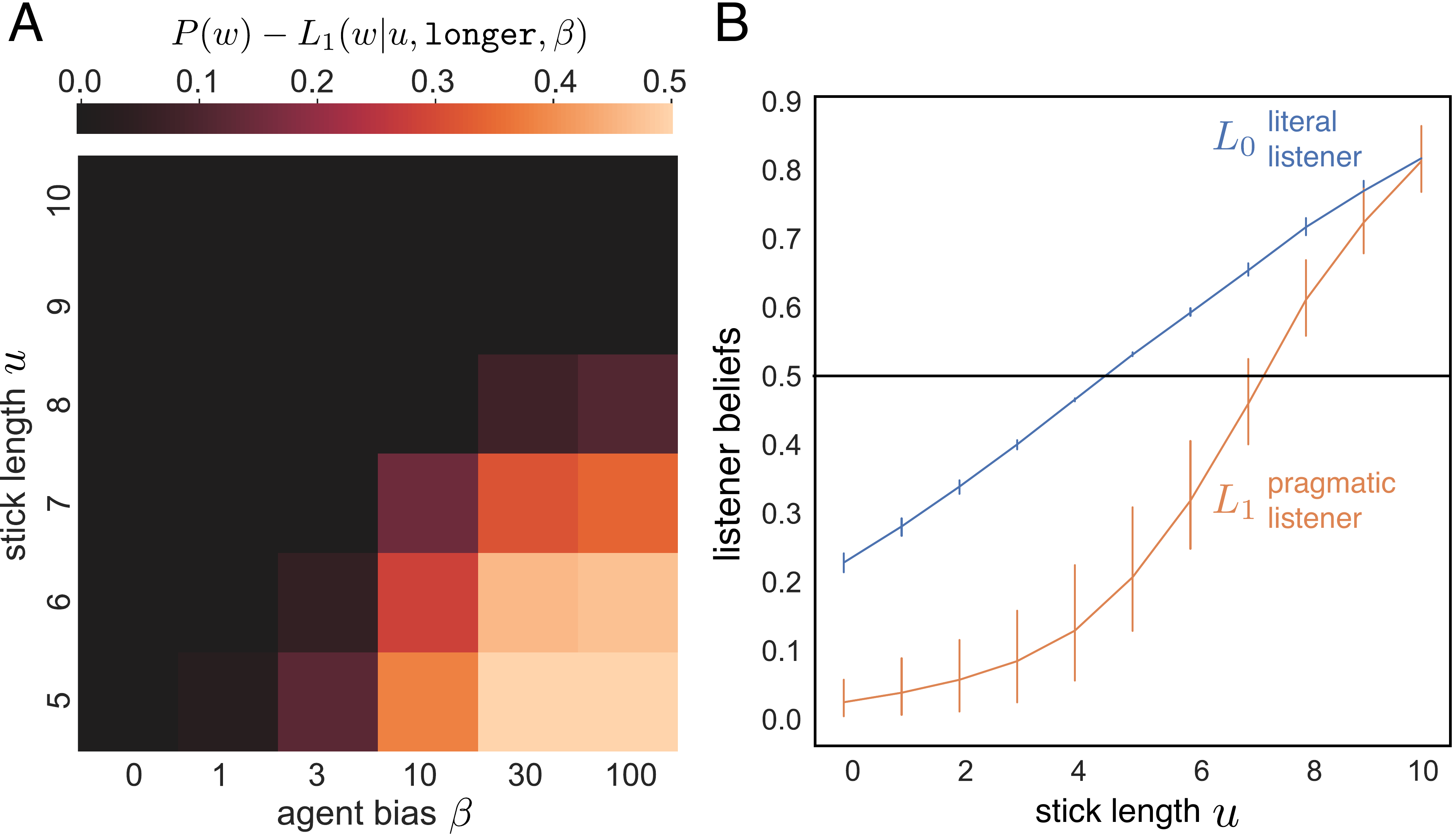}
    \caption{\emph{Model simulations}. (A) Our pragmatic listener model predicts a weak evidence effect for a broader range of evidence strengths at higher perceived speaker bias $\beta$. The color scale represents the extent to which the listener's posterior beliefs decrease in light of positive evidence, where the black region represents conditions under which no weak evidence effect is predicted. (B) Posterior beliefs of literal and pragmatic listener models as a function of evidence from long-biased speaker. Horizontal line represents prior beliefs. Error bars are given by 10-fold cross-validation across parameter fits on different subsets of our behavior data, with average $\overline{\beta} = 2.03$ and response offset $\overline{o} = -0.13$ (translating the curve down).}
    \label{fig:sim-exp}
\end{figure}

The qualitative effect observed the previous section is consistent with our pragmatic account: weak evidence only backfired for participants who expected speakers to provide the strongest available.
%Meanwhile, participants who expected speakers to provide weaker or random evidence when stronger evidence was available, drew no such inference.
%In the following section, we show that this is precisely the quantitative pattern of predictions derived from our model of social reasoning. 
In this section we conduct a series of simulations to explicitly examine the conditions under which this effect emerges from our model of recursive social reasoning between a speaker (who selects the evidence) and a listener (who updates their beliefs in light of the evidence).
Our task is naturally formalized by defining the possible utterances $u\in \mathcal{U}$ as the possible lengths of individual sticks the speaker must choose between, the world state $w$ as the true set of sticks, and the persuasive goals $w^*  \in \{\texttt{longer}, \texttt{shorter}\}$ as a binary proposition corresponding to each speaker's incentive.
Because the speaker only has access to true utterances, all utterances have equal epistemic utility (i.e. the speaker must show one of the five \emph{actual} sticks,\footnote{For related tasks studying outright lying, see \citet{oey_designing_2019, ransom2017cognitive,franke2020strategies} and \citet{oey2021lies}. For a more comprehensive and multidisciplinary overview of varieties of deception and misleading, see \citet{saul2012lying,meibauer2019oxford}.} which has the epistemic effect of reducing uncertainty about the identity of exactly one stick).
Hence, the combined utility (Eq.~\ref{eq:combU}) simplifies to the following:
\begin{align}
    S(u \mid w, w^*, \beta) & \propto \exp\{\alpha \cdot \beta \cdot \ln L_0(w^* \mid u)\}\label{eq:simplified}
\end{align}
and the persuasive utility of an utterance is monotonic in the stick length (see Appendix C for complete proofs).
Note that when ${\beta = 0}$, the pragmatic listener $L_1$ expects the speaker preferences to be uniform over true evidence, $S_1(u \mid w, w^*, \beta=0) = \textrm{Unif}(u)$, thus reducing to the literal listener $L_0$.
When $\beta \rightarrow \infty$, the pragmatic listener expects the speaker to maximize utility and choose the single strongest piece of evidence.\footnote{Because the product $\alpha \cdot \beta$ is non-zero only if the persuasion weight $\beta$ is non-zero, these two parameters are redundant in our task. We thus treat their product as a single free parameter, effectively fixing $\alpha = 1$. It is possible that a near-zero $\alpha$ (e.g. low effort from participants) may make it difficult to empirically detect a non-zero $\beta$ term in our model comparison below, but this would work against our hypothesis.}

In our simulations, we present the listener models with different pieces of evidence $u \in \{5,6,7,8,9,10\}$ and manipulate  $\beta$, which represents the degree to which the pragmatic listener $L_1$ expects the speaker $S$ to be motivated to show data that prefers target goal state $w^*=\texttt{longer}$ (the case for $\texttt{shorter}$ is analogous).
We operationalize the size of the weak evidence effect as the decrease in belief for a proposition given positive evidence supporting that proposition.
For example, if observing a stick length of 6'' \emph{decreased} the listener's beliefs that the sample was longer than 5'' from a prior belief of $P(\texttt{longer}) = 0.5$ to a posterior belief of $P(\texttt{longer} \mid u = 6)=0.4$, then we say the size of the effect is $0.5-0.4=0.1$. 

First, we observe that when $\beta = 0$ (Fig.~\ref{fig:sim-exp}A, left-most column), no weak evidence effect is observed: the listener interprets the evidence literally.
However, as the perceived bias of the speaker increases, we observe a weak evidence effect emerge for shorter sticks. 
When the perceived bias grows large (e.g. $\beta = 100$, right-most column), the weak evidence effect is found over a broad range of evidence: if the listener expects the speaker to show the single strongest piece of evidence available, then even a stick of length 8'' rules out the existence of any stronger evidence, shifting the possible range of sticks in the sample.
To further understand this effect, we computed the beliefs of literal ($J_0$) and pragmatic ($J_1$) listener models as a function of the evidence they've been shown (Fig.~\ref{fig:sim-exp}B).
While the literal listener predicts a near-linear shift in beliefs as a function of positive or negative evidence, the pragmatic listener yields a sharper S-shaped curve reflecting more skeptical belief updating.

\subsection{Quantitative Model Comparison}

Our behavioral results suggest an important role for speaker expectations in explanations of the weak evidence effect, and our simulations reveal how a pragmatic listener model derives this effect from different expectations about speaker bias.
In this section, we compare our model against alternative accounts by fitting them to our empirical data (see Appendix E for details). 

\paragraph{Fitting the RSA model to behavioral data}

We considered several variants of the RSA model, which handled the relationship between the speaker and listener phase in different ways. 
The simplest variant, which we call the \emph{homogeneous} model, assumes the entire population of participants is explained by a pragmatic model ($z=L_1$) with an unknown bias.
It is homogeneous because the same model is assumed to be shared across the whole population.
The second variant, which we call the \emph{heterogeneous} model, is a mixture model where we predicted each participant's response as a convex combination of the $J_0$ and $J_1$ models with mixture weight $p_z$ (i.e. marginalizing out latent assignments $z_i$).
In the third variant, which we call the \emph{speaker-dependent} model, we explicitly fit different mixture weights depending on the participant's response in the speaker expectations phase.
Rather than learning a single mixture weight for the entire population, this variant learns independent mixture weights for different sub-groups $z_{j}$, defined by the different sticks $j$ that participants chose in the speaker phase.
This model asks whether conditioning on speaker data allows the model to make sufficiently better predictions about the listener data.

\paragraph{Fitting anchor-and-adjust models to empirical data}

The most prominent family of \emph{asocial} models accounting for the weak evidence effect are \emph{anchor-and-adjust} (AA) models.
In these models, individuals compare the strength of new evidence $u$ against a reference point $R$ and adjust their beliefs $P(w|u)$ up or down accordingly:
\begin{equation}\label{eq:aa}
    P(w|u) = P(w) + \eta \cdot (s(u) - R),
\end{equation}
where $s(u)$ is the strength of the evidence, and $\eta$ is an adjustment weight. 
In the simplest variant \citep{hogarth1992order}, the reference point and scaling are fixed to a neutral baseline~${\eta=P(w)=1-P(w)=.5}$ and~${R=0}$.
In a more complex variant, beliefs are not updated from a \emph{neutral} baseline but instead relative to more stringent level known as the argument's ``minimum acceptable strength'' \citep[MAS;][]{mckenzie_when_2002}, which is treated as a free parameter:  $R\sim \text{Unif}[-1, 1]$.
In this case, positive evidence that falls short of $R$ may nonetheless be treated as negative evidence and decrease the listener's beliefs. 
Although the anchor is typically taken to be a specific earlier observation, it may be interpreted in the single-observation case as the participant's implicit or imagined expectations from the task instructions and cover story.
Prior work using anchor-and-adjust models would not predict a relationship between behavior in the speaker phase and in the listener phase.
We thus evaluated a homogeneous \emph{AA} model, a homogeneous \emph{MAS} model, and a heterogeneous mixture model predicting responses as a convention combination of the two.

\begin{table}[t]
    \captionsetup{font={small,stretch=.75}, strut=on}
    \centering
    \begin{tabular}{c|c||c||c||c}
        Model & Variant & Likelihood
                        & WAIC 
                        & PSIS-LOO \\
        \hline \hline
        A\&A & Homogeneous & -28.1 & 57.7 $\pm$ 9.9 & 28.8 $\pm$ 9.9 \\
        MAS & Homogeneous & 8.2 & -13.3 $\pm$ 9.6 & -6.6 $\pm$ 9.6 \\
         & Heterogeneous & 8.2 & -11.3 $\pm$ 9.5 & -5.6 $\pm$ 9.5 \\
        %  & Speaker-dependent & \textbf{-9.6} & 1.3 & -233.8 & 11.6 \\
        RSA & Homogeneous & 8.1 & -13.3 $\pm$ 9.5 & -6.7 $\pm$ 9.5 \\
         & Heterogeneous & 8.1 & -10.5 $\pm$ 9.3 & -5.2 $\pm$ 9.3 \\
         & Speaker-dependent & \textbf{12.0} & \textbf{-16.4} $\pm$ 9.1 & \textbf{-9.2} $\pm$ 9.1
    \end{tabular}
    \vspace{1em}
    \caption{Results of the model comparison, including the likelihood achieved by the best-fitting model as well as the WAIC, and PSIS-LOO ($\pm$ standard error), which penalize for model complexity.}
    \label{tab:model_comparison}
\end{table}

\paragraph{Comparison Results} We examined several metrics to assess the relative performance of these models.\footnote{All models were implemented in WebPPL \citep{dippl}; code for reproducing these analyses is available at \texttt{\href{https://github.com/s-a-barnett/bayesian-persuasion}{https://github.com/s-a-barnett/bayesian-persuasion}}.}
First, as an absolute goodness of fit measure, we found the parameters that maximized the model likelihood (see Table~\ref{tab:model_comparison}).
As a Bayesian alternative, which penalizes models for added complexity, we also considered a measure using the full posterior,\footnote{We drew 1,000 samples from the posterior via MCMC across four chains, with a burn-in of 7,500 steps and a lag of 100 steps between samples.} the Watanabe-Akaike (or Widely Applicable) Information Criterion \citep{watanabe2013widely, gelman2013bayesian}. 
The WAIC penalizes model flexibility in a way that asymptotically equates to Bayesian leave-one-out (LOO) cross-validation \citep{acerbi2018bayesian, gelman2013bayesian}, which we also include in the form of the PSIS-LOO measure \citep[PSIS stands for Pareto Smoothed Importance Sampling, a method for stabilizing estimates][]{vehtari2017practical}. 
These comparison criteria (Table~\ref{tab:model_comparison}) suggest that the added complexity of the speaker-dependent RSA model is justified: it outperforms all \emph{asocial} variants. 
For this speaker-dependent model, we found a maximum \emph{a posteriori} (MAP) estimate of $\hat{\beta}=2.26$, providing strong support for a non-zero persuasive bias term.
We found that the pragmatic  $J_1$ model best explained the judgments of participants who expected the strongest evidence to be shown during the speaker phase (mixture weight $\hat{p}_{z} = 0.99$) while the literal $J_0$ model best explained the judgments of participants who expected weaker sticks to be shown (mixture weight $\hat{p}_{z} = 0.1$).
Full parameter posteriors are shown in Fig.~S5.

\section{Discussion}

Evidence is not a direct reflection of the world: it comes from somewhere, often from other people.
Yet appropriately accounting for social sources of information has posed a challenge for models of belief-updating, even as increasing attention has been given to the role of pragmatic reasoning in classic phenomena.
In this paper, we formalized a pragmatic account of the \emph{weak evidence effect} via a model of recursive social reasoning, where weaker evidence may backfire when the speaker is expected to have a persuasive agenda.
This model critically predicts that individual differences in the weak evidence effect should be related to individual differences in how the speaker is expected to select evidence. 
We evaluated this qualitative prediction using a novel behavioral paradigm -- the Stick Contest -- and demonstrated through simulations and quantitative model comparisons that our model uniquely captures this source of variance in judgments.

Several avenues remain important for future work.
First, while we focused on the initial judgment as the purest manifestation of the weak evidence effect, subsequent judgments are consistent with the \emph{order effects} that have been the central focus of previous accounts \citep[see Appendix B;][]{anderson1981foundations, davis1984order, trueblood2011quantum}.
Thus, we view our model of social reasoning as capturing an orthogonal aspect of the phenomenon, and further work should explicitly integrate computational-level principles of social reasoning with process-level mechanisms of sequential belief updating. 
Second, our model provides a foundation for accounting for related \emph{message involvement} effects (e.g., emotion, attractiveness of source), \emph{presentation} effects (e.g. numerical vs. verbal descriptions), and \emph{social affiliation} effects (i.e., whether the source is in-group) that have been examined in real-world settings of persuasion \citep[e.g.][]{martire2014interpretation,debono1988source,bohner2002expertise,park2007effects,cialdini1993influence, falk_persuasion_2018}, 
These settings also involve uncertainty about the \emph{scale} of possible argument strength, unlike the clearly defined interval of lengths in our paradigm.
Third, while the weak evidence effect emerges after a single level of social recursion, it is natural to ask what happens at higher levels: what about a more sophisticated speaker who is \emph{aware} that weak evidence may lead to such inferences?
Our paradigm explicitly informed participants of the speaker bias, but uncertainty about the speaker's hidden agenda may give rise to a \emph{strong} evidence effect \citep{perfors_stronger_2018}, where speakers are motivated to \emph{avoid} the strongest arguments to appear more neutral (see Appendix E).
Based on the self-explanations we elicited (Table S2), it is possible that some participants who expected less strong evidence were reasoning in this way. 
These individual differences are consistent with prior work reporting heterogeneity in levels of reasoning in other communicative tasks \cite[e.g.][]{franke2016reasoning}. 

%It is hypothesized that this arises due to strong evidence making listeners believe that the communicator was biased, hence leading to the evidence being discounted. Future experiments would seek to use the speaker models and data given by the stick task in order to investigate the presence of this effect.

We used a within-participant individual differences design for simplicity and naturalism, but there are also limitations associated with this design choice.
For example, it is possible that the group of participants who expected weaker evidence to be shown first could be systematically different from the other group in some way, such as differing levels of inattention or motivation, that explains their behavior on \emph{both} speaker and listener trials.
We aimed to control for these factors in multiple ways, including strict attention checks (Appendix A) and self-explanations (Tables S2-S3), which suggest a thoughtful rationale for expecting weaker evidence.
However, an alternative solution would be to explicitly manipulate  social expectations about the speaker in the cover story (e.g. training participants on speakers that tend to show weaker or stronger evidence first).
Such a design would license stronger causal inferences, but would also raise new concerns about exactly what is being manipulated. 
A second limitation of our design is that the speaker phase was always presented before the listener phase. 
It is already known that the order of these roles may affect participants' reasoning \cite[e.g.][]{sikos2021speak,shafto_rational_2014}, but asocial accounts of the weak evidence effect would not predict any relationship between speaker and listener trials under \emph{either} order. 
Hence, we chose the order we thought would minimize confusion about the task; it is not our goal to suggest that social reasoning is spontaneous or mandatory, and we expect that social-pragmatic factors may be more salient in some contexts than others \cite[e.g. when evidence is presented verbally vs. numerically, as in][]{martire2014interpretation}.

Probabilistic models have continually emphasized the importance of the data generating process, distinguishing between assumptions like \emph{weak} sampling, \emph{strong} sampling, and \emph{pedagogical} sampling  \citep{hsu2009differential, shafto_rational_2014, tenenbaum1999bayesian, tenenbaum2001generalization}. 
Our work considers a fourth sampling assumption, \emph{rhetorical sampling}, where the data are not necessarily generated in the service of pedagogy but rather in the service of persuasive rhetoric.
Critically, although we formalized this account in a recursive Bayesian reasoning framework, insights about rhetorical sampling are also compatible with other frameworks: for example, work in the anchor-and-adjust framework may use similar principles to derive a relationship between information sources and reference points. 
% would be to provide the computational underpinnings of a \emph{rhetorical principle}, by means of analogy to Grice's cooperative principle \citep{grice_logic_1975}, which would describe how people interact in social settings in which they aim not only to inform but also to \emph{persuade}.
Such socially sensitive objectives may be particularly key in the context of developing artificial agents that are more closely aligned with human values \citep{irving_ai_2018,carroll2019utility, hilgard_learning_2019}.
As we navigate an information landscape increasingly filled with disinformation from adversarial sources, a heightened sense of skepticism may be rational after all.%, in which agents compete in a debate game to produce the most true, useful information for a human to judge in a decision-making task. 
%While previous models have looked at agents playing a zero-sum game using Monte Carlo Tree Search, it is  plausible that an agent model with a theory of mind about the target of persuasion can perform better in terms of producing useful information. This would be in line with other research indicating the importance of human-like AI models for tasks in which humans and AI systems must cooperate \citep{}.

\acknowledgments
This work was supported by grant \#62220 from the John Templeton Foundation to TG.
RDH is funded by a C.V. Starr Postdoctoral Fellowship and NSF SPRF award \#1911835.
We are grateful for early contributions by Mark Ho and helpful conversations with other members of the Princeton Computational Cognitive Science Lab, as well as Ryan Adams and members of the Laboratory for Intelligent Probabilistic Systems. 

\bibliography{refs.bib}

\renewcommand{\thefigure}{S\arabic{figure}}
\renewcommand{\thetable}{S\arabic{table}}
\setcounter{table}{0}
\setcounter{figure}{0}

\section*{Appendix A: Exclusions and attention checks}

Our pre-registered exclusion criteria used two basic attention checks. 
First, participants were required to complete a comprehension quiz immediately following the task instructions, and we excluded participants who failed to successfully complete this quiz within three attempts.
Second, at the end of the experiment, we asked participants to use a slider to indicate the degree of bias they believed each contestant exhibited. 
These motivations were stated explicitly in the instructions (e.g. ``the red contestant will receive \$10 if the judge chooses ``shorter,'' otherwise the blue contestant will receive \$10'') so, although participants may differ in the \emph{degree} to which they thought such incentives would bias the contestants away from neutrality, we took responses in the \emph{opposite} direction of the incentive as indicative of inattentiveness or misunderstanding of task instructions. 

We therefore coded bias check responses as ``incorrect'' if the slider response was inconsistent with the bias given in the instructions (e.g. if the short-biased contestant received a slider rating above the midpoint, $s \ge 50-\epsilon$, or the long-biased contestant received a slider rating below the midpoint, $s \le 50+\epsilon$ where we set $\epsilon = 5$ to allow for the possibility of motor jitter from participants who intended to use the exact midpoint.)
In our pre-registered second sample (reported in the main text), $793$ participants completed instructions and $723$ (91\%) of them passed the attention check.

While these pre-registered criteria were designed to ensure that apparent differences in speaker and listener behavior were not simply driven by general attentional factors, it is possible that participants who did not expect the strongest evidence to be shown in the speaker phase (238 participants, or 33\%) were still systematically less attentive than other participants. 
To address this concern, we analyzed a series of other measures to assess the degree of attention and task understanding across ``speaker expectation'' groups.
Specifically, we examine internal consistency within several post-test questions, where we asked participants (i) to make a final two-alternative forced choice verdict about whether the sample of sticks is `longer' vs. `shorter' than 5 inches, (ii) to provide a point estimate of their best guess of the actual mean on a slider ranging from 1 inch to 9 inches, and (iii) to guess the values of the remaining three sticks that were not revealed, allowing us to impute a ``generative'' average across the two observed values and the three guessed values (Table \ref{tab:attentionchecks}).

\begin{table}[t]
\centering
\begin{tabular}{ll|p{.2\linewidth}p{.165\linewidth}}
  \multirow{2}{*}{group} & \multirow{2}{*}{n} & both 2AFC and point estimate consistent & generative model also consistent  \\ 
  \hline
 strongest first & 485 & 0.97 & 0.89 \\
 \emph{not} strongest first & 238 & 0.96 & 0.86\\
\end{tabular}
    \captionsetup{font={small,stretch=.75}, strut=on}
\caption{Stricter attention check passage rates broken out by speaker group.}
\label{tab:attentionchecks}
\end{table}
We say a participant passed the 2AFC check if their binary verdict (`longer' vs. `shorter') is consistent with the direction of their point estimate.
We say a participant also passed the stricter ``generative'' check if the average imputed from their guesses for the remaining three unobserved sticks matches their 2AFC and point estimates.
We observe that rates for the these stricter checks were somewhat lower for participants who expected speakers not to show the strongest evidence first (97\% vs. 96\%, and 89\% vs. 86\%, respectively), though neither of these differences was significant, $\chi^2(1)=0.76, p=0.38$ and $\chi^2(1)=1.05, p=0.31$, respectively. 
Rates were far above chance for all groups.
To ensure robustness, we re-ran our primary analyses on the subset of participants that passed the strictest conjunction of all checks, which is highly improbable under an inattentive null model, and obtained nearly identical results (most crucially, a significant interaction, $t(718)=5.18, p<0.001$).

\begin{figure}
    \centering
        \captionsetup{font={small,stretch=.75}, strut=on}

    \includegraphics[scale=0.8]{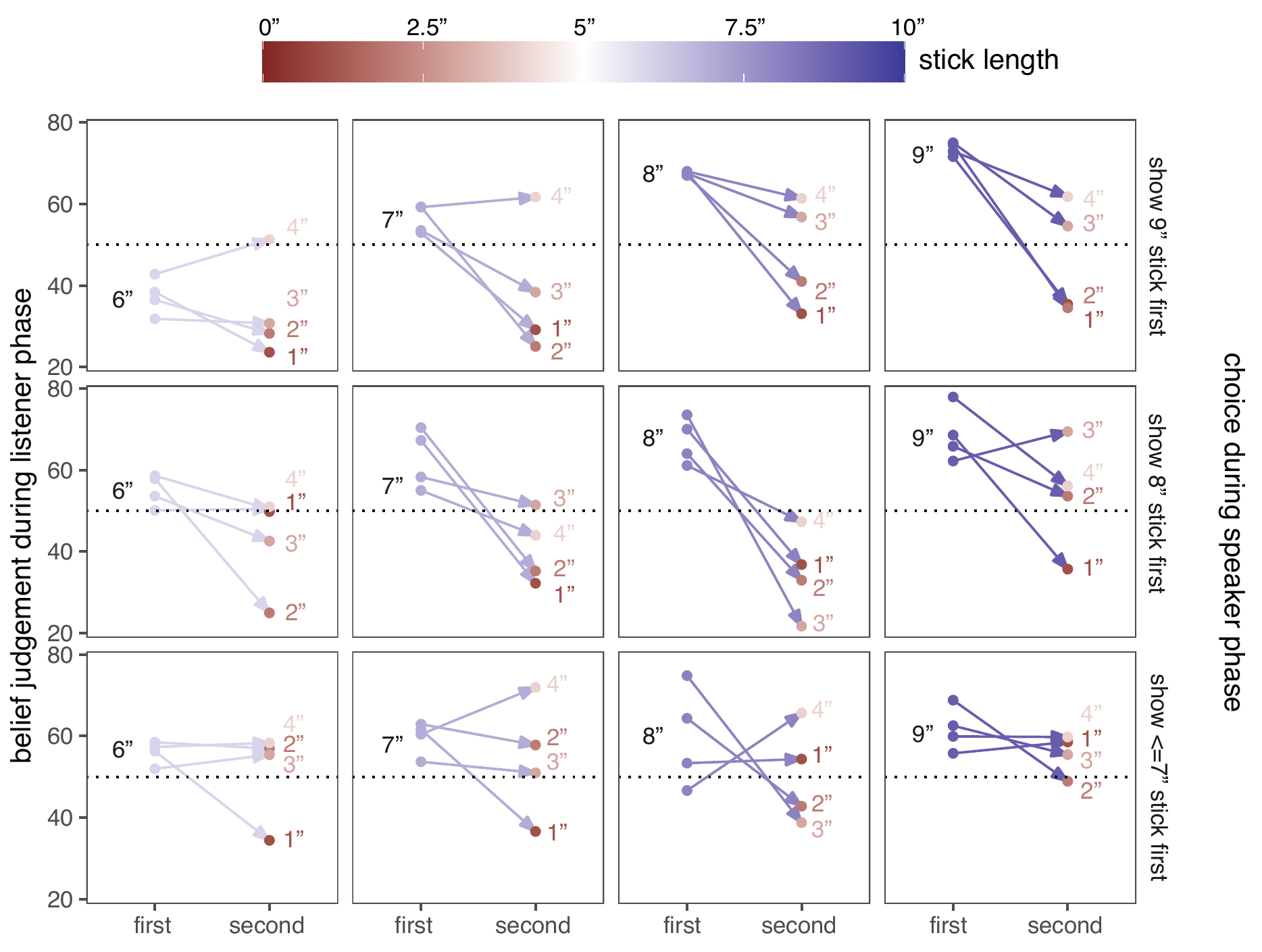}
    \caption{Participants revised their beliefs after obtaining a second piece of evidence. Each facet represents participants who were given the same initial piece of evidence (blue dots) with each arrow connecting their judgment after the first piece of evidence and the second piece of evidence. In most cases, participants revised their estimates down, although participants who showed a weak evidence effect for the first stick (top column) also displayed a classical weak evidence effect on the second piece of evidence (e.g. in the second row, participants who saw a 7" stick on the first trial were slightly \emph{more} confident the average was longer after seeing a 4" stick).}
    \label{fig:order_effect_detailed}
\end{figure}

\begin{figure}
    \captionsetup{font={small,stretch=.75}, strut=on}
    \centering
    \includegraphics[scale=0.8]{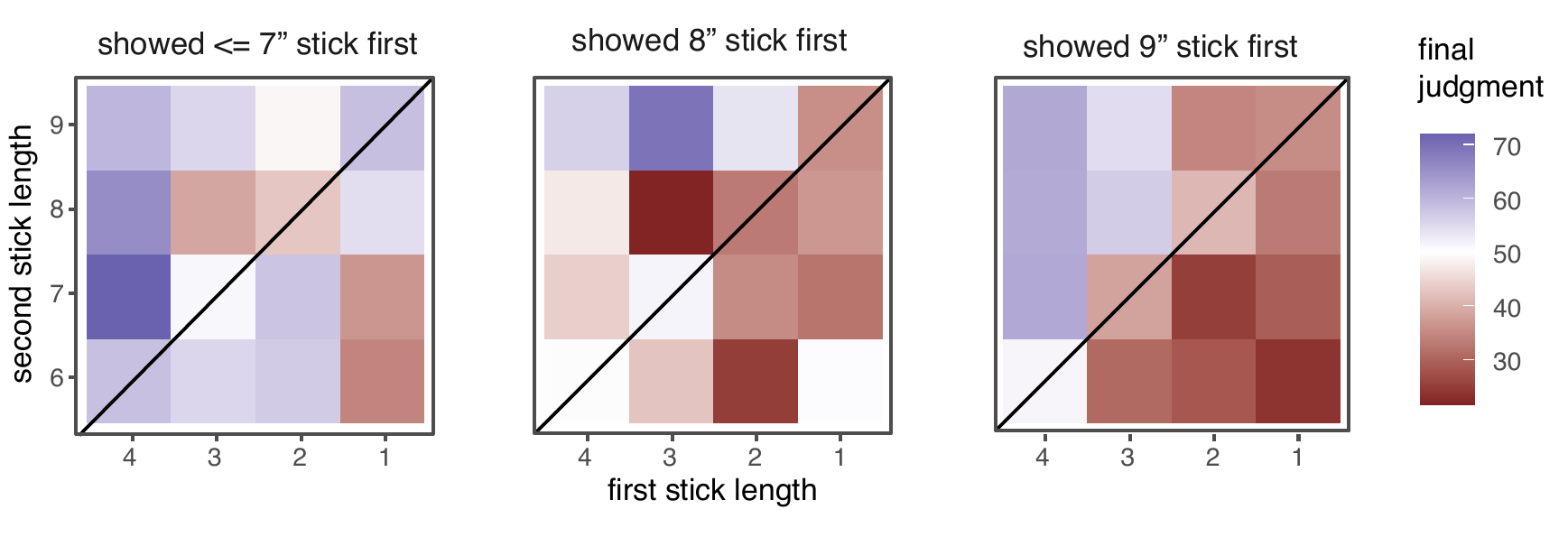}
    \caption{We found strong order effects, with the belief judgment elicited after the second stick apparently affected by a recency bias. Under perfect averaging, the diagonal would leave the judge with complete uncertainty (denoted on our color scale by white), since the evidence from both the longer side (blue) and the shorter side (red) should cancel out.}
    \label{fig:order_grid}
\end{figure}

\newpage
\section*{Appendix B: Order effects}

While we focus on the first piece of evidence as the clearest weak evidence effect, we also collected a second response after a second piece of evidence was shown by the other speaker. 
These responses are visualized in Fig.~\ref{fig:order_effect_detailed}.
As expected, we observed a recency effect (more easily observed in the diagonal of Fig.~\ref{fig:order_grid}, where evidence from the ``short''-biased and ``long''-biased speakers were equally strong), where participants weighted the second piece of evidence more strongly.

\section*{Appendix C: Proofs}

\begin{theorem}
The speaker model using the combined utility Eq.~6 simplifies to Eq.~8 for the stick contest task.
\end{theorem}

\begin{proof}
We begin by substituting the combined utility (Eq.~6) into the speaker softmax:
\begin{align}
S(u | w, w^*) & \propto \exp\{\alpha \cdot U(u; w, w^*)\}\nonumber \\
 & = \exp\{\alpha \cdot [U_{\textrm{epi}}(u; w) + \beta \cdot U_{\textrm{pers}}(u; w^*)]\} \nonumber \\
& =  \exp\{\alpha \cdot U_{\textrm{epi}}(u; w)\} \cdot  \exp\{\alpha \cdot  \beta \cdot U_{\textrm{pers}}(u; w^*)\} \nonumber
\end{align}
Now, using Eq.~3 to expand the first term, note that 
$$U_{\textrm{epi}}(u; w) = \ln P_{L_0}(w | u) = \ln  \frac{P(w)\delta_{\llbracket u \rrbracket (w)}}{\sum_w  P(w) \delta_{\llbracket u \rrbracket (w)}} =\left\{\begin{array}{ll}-\ln  N & \textrm{if } \llbracket u \rrbracket (w) \\ -\infty & \textrm{o.w.}\end{array}\right.$$ 
where $N$ is the number of sticks in the true set ($N=5$ in our experiment).
However, we already assume that the set of possible utterances $\mathcal{U}$ are the true sticks in the underlying set (i.e. the contestants cannot make up sticks, they must choose one of the $N$ sticks in the set), so
\begin{align}
\exp\{\alpha \cdot U_{\textrm{epi}}(u,w)\} & =\left\{\begin{array}{ll}\alpha / N & \textrm{if } \llbracket u \rrbracket (w) \\ 0 & \textrm{o.w.}\end{array}\right. \nonumber\\
& = \alpha / N \nonumber
\end{align}
Because all utterances have the exact same epistemic utility $U_{epi}$, this term drops out of the soft-max:
\begin{align}
\nonumber
S(u | w, w^*) & \propto \exp\{\alpha \cdot U_{\textrm{epi}}(u; w)\} \cdot  \exp\{\alpha \cdot  \beta \cdot U_{\textrm{pers}}(u; w^*)\} \nonumber \\
& \propto  \exp\{\alpha \cdot  \beta \cdot U_{\textrm{pers}}(u; w^*)\} \nonumber\\
& = \exp\{\alpha \cdot  \beta \cdot \ln L_0(w^* | u)\}\nonumber
\end{align}
yielding Eq.~8.
\end{proof}

\begin{theorem}
Persuasiveness monotonically increases as a function of stick length.
\end{theorem}

\begin{proof}
We say an utterance $u$ is more persuasive than an utterance $u'$ when $$U_{\textrm{pers}}(u \mid w^*) > U_{\textrm{pers}}(u' \mid  w^*).$$
Under the stick contest, let $\mathcal{L} = \{l_1, \dots, l_N\}$ be an partially-ordered set of $N$ stick lengths, such that $l_i\le l_j$ for any index $i < j$.
We denote the mean stick length by $\bar{l} = \frac{1}{N}\sum_i l_i$.
Without loss of generality, let the speaker's persuasive goal be $w^* = $\texttt{ shorter} $=\bar{l}<5$ (the argument follows analogously for \texttt{longer}). 
Take two utterances $u=l_i$ and $u'=l_j$ such that $l_i \le l_j$ (i.e. such that $u$ is just as short or shorter than $u'$).
First, we expand the utility:
\begin{align}
U_{\textrm{pers}}(u \mid \texttt{shorter}) & = \ln L_0(\texttt{shorter} \mid u)\nonumber \\
  & = \ln P(\bar{l} < 5 \mid l_i)\nonumber \\
  & = \ln P\left(\frac{l_i + \sum l_{-i}}{N} < 5\right)\nonumber\\
  & = \ln P\left(\sum l_{-i} < 5N - l_i\right)\nonumber
 \end{align}
Now, let $X$ be a random variable representing the sum of the $N-1$ still-unknown sticks, $X=\sum l_{-i}$.
Then we recognize this as the cumulative distribution function (CDF), $F_X(x) = P(X < x)$. 
Because the underlying set of sticks $\mathcal{L}$ is assumed to be i.i.d., note that the random variable $X=\sum l_{-i}$ does not depend on the original choice of $i$.
Critically, we know that the cumulative distribution function is monotonic increasing in $x$, i.e. $F_X(a) \le F_X(b)$ for $a\le b$.
Hence if $l_i \le l_j$ then $5N - l_i \ge 5N - l_j$ and $F_X(5N - l_i) \ge F_X(5N - l_j)$:
\begin{align}
U(u \mid \texttt{shorter}) & =  \ln P\left(\sum l_{-i} < 5N - l_i\right)\nonumber \\
& = \ln F_X(5N-l_i)\nonumber \\
& \ge \ln F_X(5N-l_j)\nonumber \\
& = U(u' \mid \texttt{shorter})\nonumber
\end{align}
\end{proof}

\section*{Appendix D: Results from original sample}

The results reported in the main text are based on a pre-registered replication we conducted during the revision of the manuscript (May 2022). 
In this appendix, we report the corresponding results from our original sample (February 2020).
The only methodological difference between the original study and the internal replication was the way we counter-balanced the order of the ``long''- vs. ``short''-biased contestants. 
In our original study, the ``long''-biased contestant always presented their evidence first; in our replication, the order of the contestants was randomized. 
Additionally, in our replication, we added the following clarification to the instructions: ``Sticks ranging in length from 1 to 9 inches are equally likely to appear in the set.''
Participants in the initial sample were recruited on the Prolific platform, with no restriction on country. 
Of the $784$ participants who successfully completed the instructions, $708$ passed the second attention check.

Our regression model was the same as in the main text, except we did not include a fixed effect of ``long'' vs. ``short'': all participants were shown evidence from the ``long''-biased speaker.
As in the study reported in the main text, we found a significant interaction between speaker expectations and evidence strength on beliefs about the underlying mean, $t(704) = 5.9, p < 0.001$. 
For participants who expected the speaker to provide the strongest evidence (421 participants or 60\% of our sample), the weak evidence provided by a six inch stick backfired, leading them to instead expect that the mean stick length was significantly less likely to be longer than five inches, $m = 37.5$, 95\% CI: $[33.1, 41.9]$, $t(98) = −5.7$, $p < 0.001$. 
Meanwhile, for participants who expected to be shown the second-longest stick (40\% of the sample), no weak evidence effect was found, with the `longest stick' group significantly different from the other groups, $t(167) = −5.5, p < 0.001$.

\section*{Appendix E: Model fitting details}

% Results available on Della 
% \begin{table}[t]
%     \centering
%     \begin{tabular}{c|c||c|c||c}
%         Model & Variant & lppd & $p_{\text{WAIC}_2}$ & $\widehat{\text{elppd}}_{\text{WAIC}}$ \\
%         \hline \hline
%         A\&A & Homogeneous & -233.0 & 158.9 & 783.9 \\
%         MAS & Homogeneous & -86.5 & 554.3 & 1281.6 \\
%          & Heterogeneous & 347.5 & 226.2 & -242.6 \\
%          & Speaker-dependent & 297.5 & 190.8 & -213.4 \\
%         \textbf{RSA} & Homogeneous & -151.9 & \textbf{125.9} & 555.6 \\
%          & Heterogeneous & 383.4 & 293.1 & -180.7 \\
%          & \textbf{Speaker-dependent S1} & 397.5 & 274.6 & \textbf{-245.7} \\
%          & Speaker-dependent S2 & \textbf{421.8} & 315.3 & -213.0 
%     \end{tabular}
%     \caption{Results of the extended model comparison, showing the computed log pointwise predictive density (lppd), the computed WAIC parameter correction formula ($p_{\text{WAIC}_2}$), and the estimated expected log pointwise predictive density that serves as the overall WAIC metric ($\widehat{\text{elppd}}_{\text{WAIC}}$). Refer to \citet{gelman2013bayesian} for definitions of these terms. In bold are the best models for each column. Note that the overall best model, the speaker-dependent model, excels precisely because of its \emph{trade-off} of predictive accuracy and model capacity, rather than it being the best in either of these categories individually.}
%     \label{tab:model_comparison_waic}
% \end{table}

\subsection{RSA model}

We used the following priors for our Bayesian data analysis:
\begin{align*}
    y & \sim \textrm{Gaussian}(\mu + o, 0.3)  \\
    p_{z} & \sim \textrm{Unif}[0,1]   \\
    \beta & \sim  \textrm{Unif}[0,10] \\
    o & \sim \textrm{Unif}[-0.5,0.5]
\end{align*}
where $p_{z}$ is the mixture weight used for heterogeneous models, $\mu = P_{J_i}(\texttt{longer} | u) \in [0,1]$ is the RSA listener model's posterior belief, and $o$ is a uniform offset included to allow for systematic response biases in use of the slider.
Intuitively, $\textrm{Gaussian}(\mu+o, 0.3)$ can be viewed as a simple way of scoring the error between the model prediction $\mu+o$ and the participant's response $y$.
For the speaker-dependent model, we used independent priors depending on the participant's choice of stick $j$:  $p_{z}^{(j)} \sim \textrm{Unif}[0,1]$.
Because there were relatively fewer participants who expected the \texttt{longer} speaker to choose 0.2 or 0.4 (sticks that were in the opposite direction of their goal; and vice versa for the \texttt{shorter} speaker), we collapsed these participants together, forming three groups: those who expected the strongest evidence to be presented first (e.g. who selected $\{0.2, 0.9\}$ for the \emph{short} and \emph{long} biased speakers, respectively), those who expected the second-strongest to be presented first (e.g. who selected $\{0.4, 0.8\}$, respectively), and those who expected less strong evidence.
However, our findings are robust to whether we collapse these groups or not. 

\subsection{Belief-adjustment models}

In the notation of \citet{mckenzie_when_2002}, Eq.~9 is written:
\begin{equation}
    C_k = C_{k-1} + w_k \cdot (s(e_k) - R),\label{eq:aa-singleobs}
\end{equation}
where $C_k \in [0, 1]$ is the degree of belief in a particular claim after being presented with evidence $e_k$, $s(e_k)$ is the \emph{independently judged} strength of evidence $e_k$, $R$ is a reference point, and $w_k\in [0, 1]$ is an adjustment weight for evidence $e_k$. 
In the \emph{adding} variant of the belief-adjustment model, \citet{hogarth1992order} argue that the evidence should be encoded in an absolute manner, letting ${R=0}$ and $s(e_k)\in [-1,1]$, and  assuming that if ${s(e_k)\le R}$ then ${w_k=C_{k-1}}$, otherwise ${w_k=1-C_{k-1}}$.\footnote{
The \emph{averaging} variant, in which evidence is encoded in relationship to the current belief in the hypothesis, is more suited for \emph{estimation} tasks involving some kind of moving average \citep{hogarth1992order}, whereas the Stick Contest is better described as an \emph{evaluation} task in which a single hypothesis is under consideration (``is the sample long?''). We also found empirically that the adding variant provided a better fit to the data than the averaging variant.}
To allow the reference point for evidence to be more demanding than neutrality, \citet{mckenzie_when_2002} proposed replacing the reference point $R$ with a Minimum Acceptable Strength (MAS) threshold $(m \mid e)$, that depends on the evidence previously presented. We can therefore rewrite~Eq.~\ref{eq:aa-singleobs} as
\begin{equation}
    C_k = C_{k-1} + w_k \cdot (s(e_k) - (m_k \mid e_1, ..., e_{k-1})).
\end{equation}

To fit this class of models to our data, we follow \citet{trueblood2011quantum}, assuming a mapping between stick length and evidence strength given by a centered logistic function:
\begin{equation}
    \text{strength}(u) = \frac{1}{1 + \exp{(-B\cdot (u - 5))}} - 0.5,
\end{equation}
where the logistic growth rate $B$ is fit to the data (we used a uniform prior $B\sim \textrm{Unif}[0,10])$.
This function satisfies several desiderata: it is monotonically increasing in the size of the stick, it is bounded in the interval $[-1, 1]$, and it is centered in line with the prior over stick lengths, so that a stick of length 5 inches has a strength of 0.5.

For the anchor-and-adjust (AA) variant, we fix the reference point as~${R=0}$, and for the minimum acceptable strength (MAS) variant, we infer a reference point with prior $R\sim \textrm{Unif}[-1, 1]$.
We consider \emph{homogeneous} variants in which the entire population is assumed to share the same model with the same parameters, as well as a \emph{heterogeneous} model, in which we assume \emph{a priori} that participants are a convex combination of the two models. 
As in the RSA models, we infer the mixture weight $p_z$ that best explains the population-level mixture (marginalizing over latent variable assignments $z$).

\subsection{Higher levels of reasoning and the strong evidence effect}

While our cover story explicitly provided participants with the motivations of speakers, in terms of their financial incentives, these motivations are less obvious in most real-world scenarios.
They must be \emph{inferred} from what the speaker is saying.
This is straightforwardly derived in our framework by allowing the listener to jointly infer the true state of the world $w$ \emph{and} the speaker's bias $\beta$:
\begin{equation}
    P_{L_1} (w, \beta \mid u) \propto P_{S_1} (u \mid w, \beta) \cdot P(w)
\end{equation}

Our formulation raises a natural question about how speakers would behave if they were \emph{aware} judges were making such inferences.
This emerges at the next level of recursive reasoning:
\begin{equation}
    P_{S_2} (u \mid w, \beta) \propto \exp{\biggl(\lvert\beta\rvert \ln\bigl( P_{L_1} (w^* \mid u) - w_c \cdot C(u) \bigr)\biggr)}.
\end{equation}
where $C(u)$ represents some cost associated with being perceived as biased by the judge:
\begin{equation}
    C(u)=\mathbb{E}_{\beta\sim P_{L_1}(\cdot \mid u)} \bigl[\lvert \beta \rvert\bigr],
\end{equation}
and~${w_c\ge 0}$ is a parameter specifying the degree of the cost. 
We included a $J_2$ model who reasons about this listener in our model comparison (i.e. allowing participants to be explained by a convex combination of all three levels) and found that this three-level speaker-dependent model leads to improved performance over the two-level speaker-dependent model (max likelihood $=16.2$, WAIC$= -18.3\pm 8.9$, PSIS-LOO$= -9.2\pm 8.9$.)
We conjecture that this formulation is required to account for the \emph{strong evidence effect} \citep{perfors_stronger_2018}, in which the desire to appear unbiased leads a speaker to choose weaker evidence in spite of the presence of stronger alternatives, but leave further investigation for future work. 

\section{Appendix F: Transcript of the Experiment}
The written instructions for our experiment are reproduced below. Note that the task can be seen exactly as participants experienced it (e.g. with images) using the code released in our repository: \texttt{\href{https://github.com/s-a-barnett/bayesian-persuasion}{https://github.com/s-a-barnett/bayesian-persuasion}}.
\begin{quotation}
In this task, you will serve as the judge for a heated game between these two contestants. The two contestants in this game have been given a set of sticks ranging in length from very long ones to very short ones.  Sticks ranging in length from 1 to 9 inches are equally likely to appear in the set. One contestant (shown in pink) will be rewarded handsomely if they can convince you that the average length of these sticks is shorter than 5in (see dotted line). The other (shown in blue) will get paid if if they can convince you that the average length of these sticks is longer than 5in (see dotted line). In this case, the average length is 6in, so the position that this person was arguing for was true. As the judge, however, you will not be able to see the full set of sticks: you will only see what the contestants choose to show you. They will each get to show exactly one of the five sticks to convince you. After you see each stick, you will use this slider to report how strongly you are leaning in your decision. If you think the stick average is more likely to be shorter than 5in, click further to the left. If you think it is more likely to be longer than 5in, click further to the right.
\end{quotation}

\begin{figure}[t]
    \captionsetup{font={small,stretch=.75}, strut=on}

    \centering
    \includegraphics[width=0.8\linewidth]{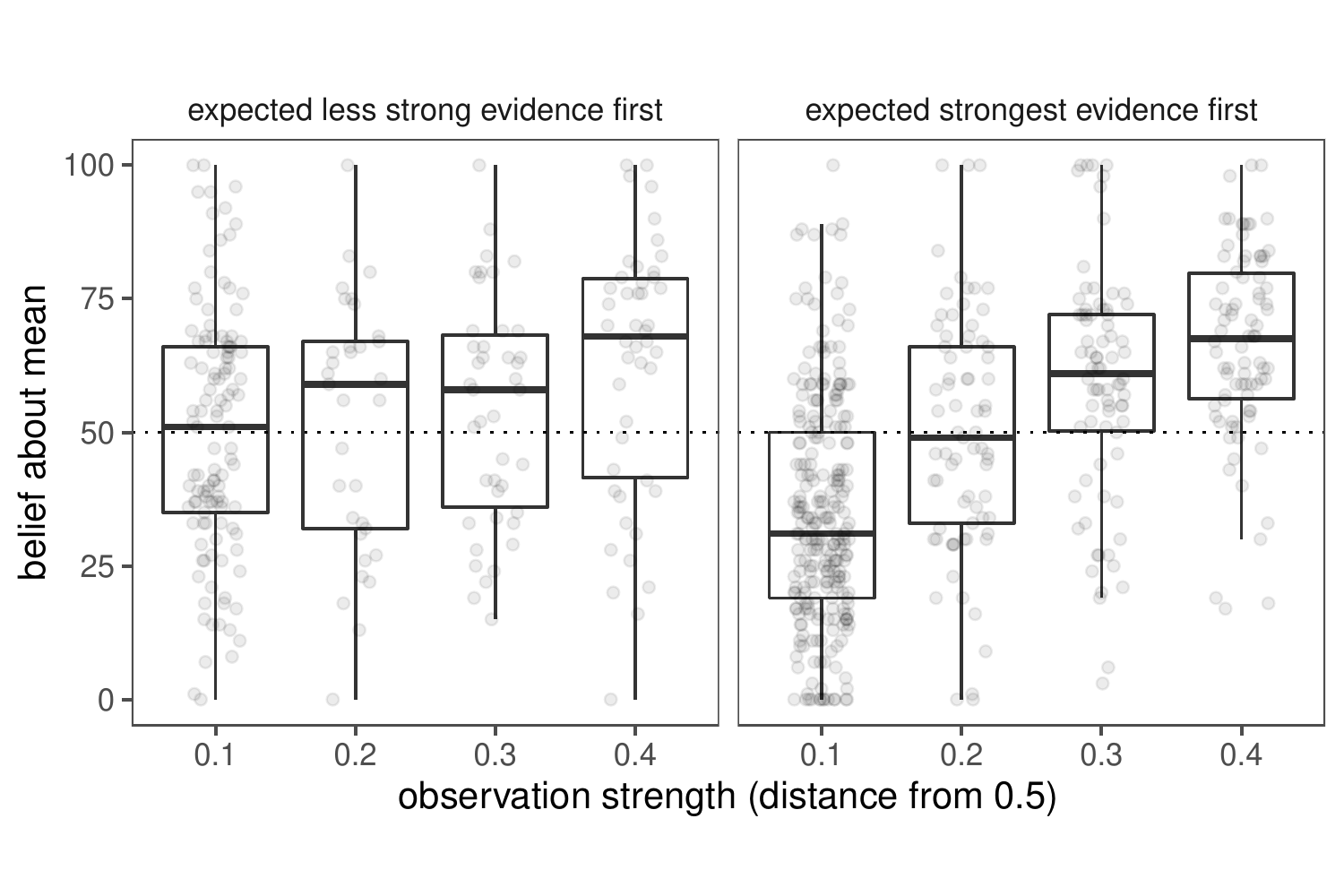}
    \caption{The raw data distribution of responses for the listener phase, where each individual (jittered) point is a different participant and the boxplot represents the median (dark line) and first and third quartiles (top and bottom of box) of the response distribution.}
    \label{fig:fig2_raw}
\end{figure}

\begin{figure}[th]
    \captionsetup{font={small,stretch=.75}, strut=on}

    \centering
    \includegraphics[width=0.75\linewidth]{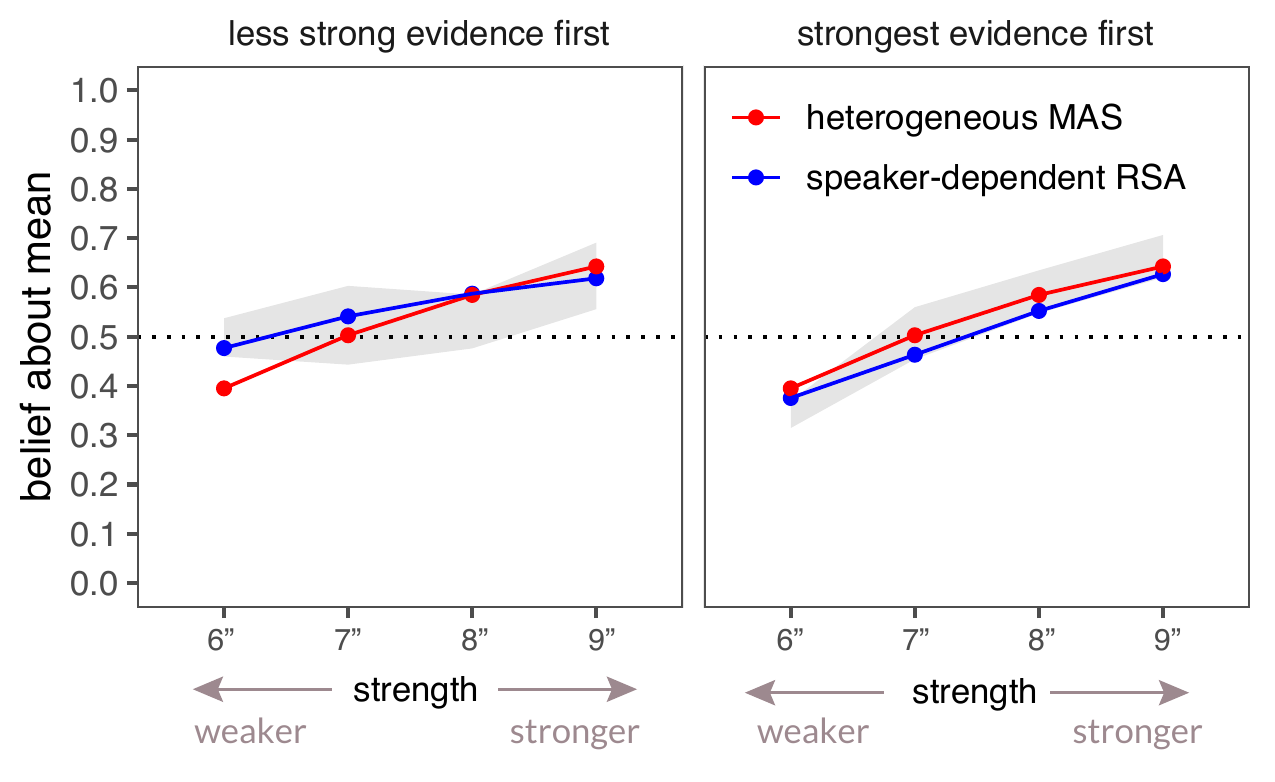}
    \caption{We visualized the posterior predictives for the speaker-dependent RSA model (blue) and heterogeneous MAS model (red). The facets represent which stick was expected to be chosen first in the speaker phase, and the grey region represents the 95\% confidence interval of the empirical data.
    }
    \label{fig:ppcheck}
\end{figure}

\begin{figure}
    \centering
        \captionsetup{font={small,stretch=.75}, strut=on}

        \centering
        \includegraphics[width=0.9\linewidth]{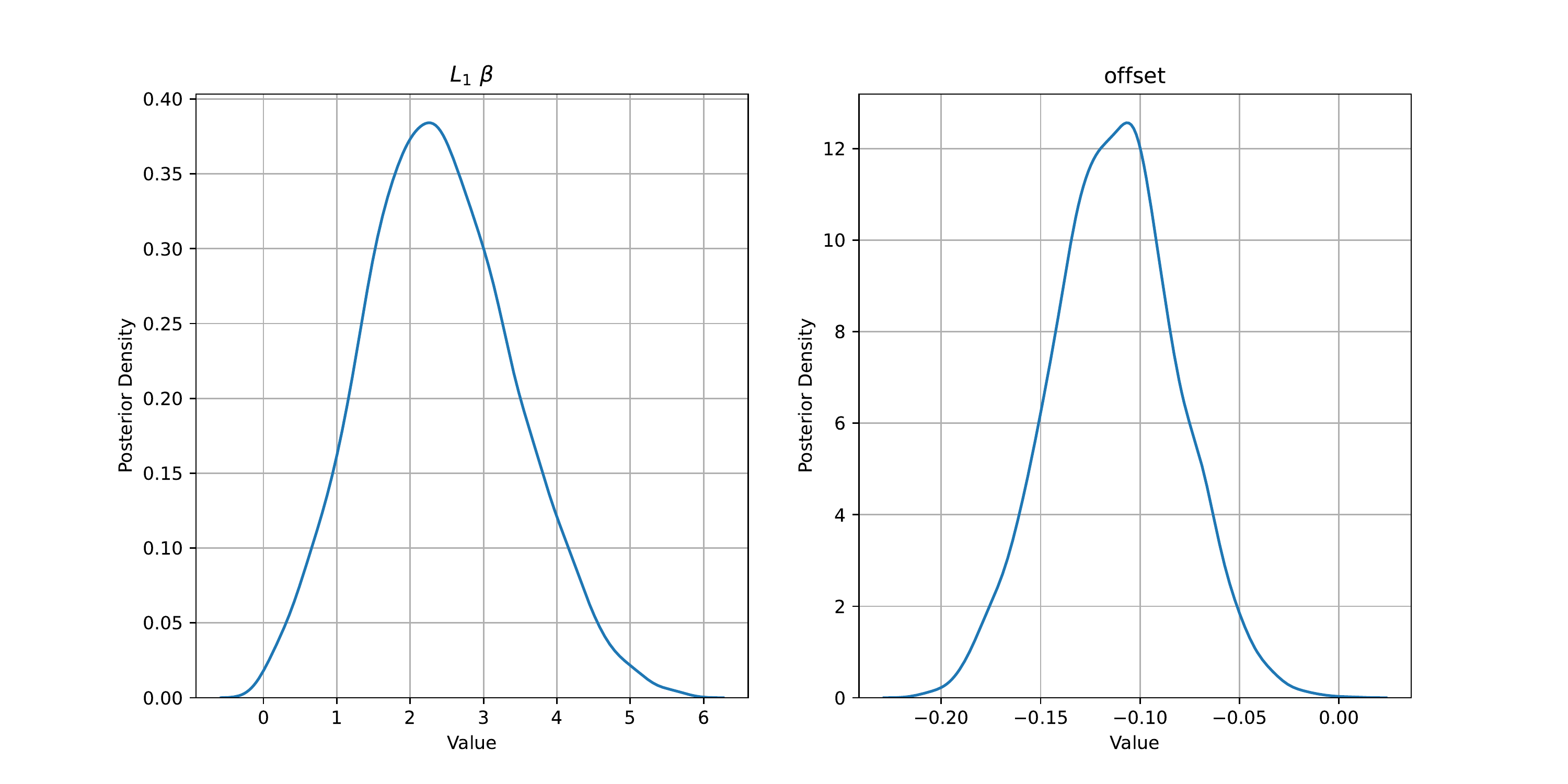}
        \includegraphics[width=0.8\linewidth]{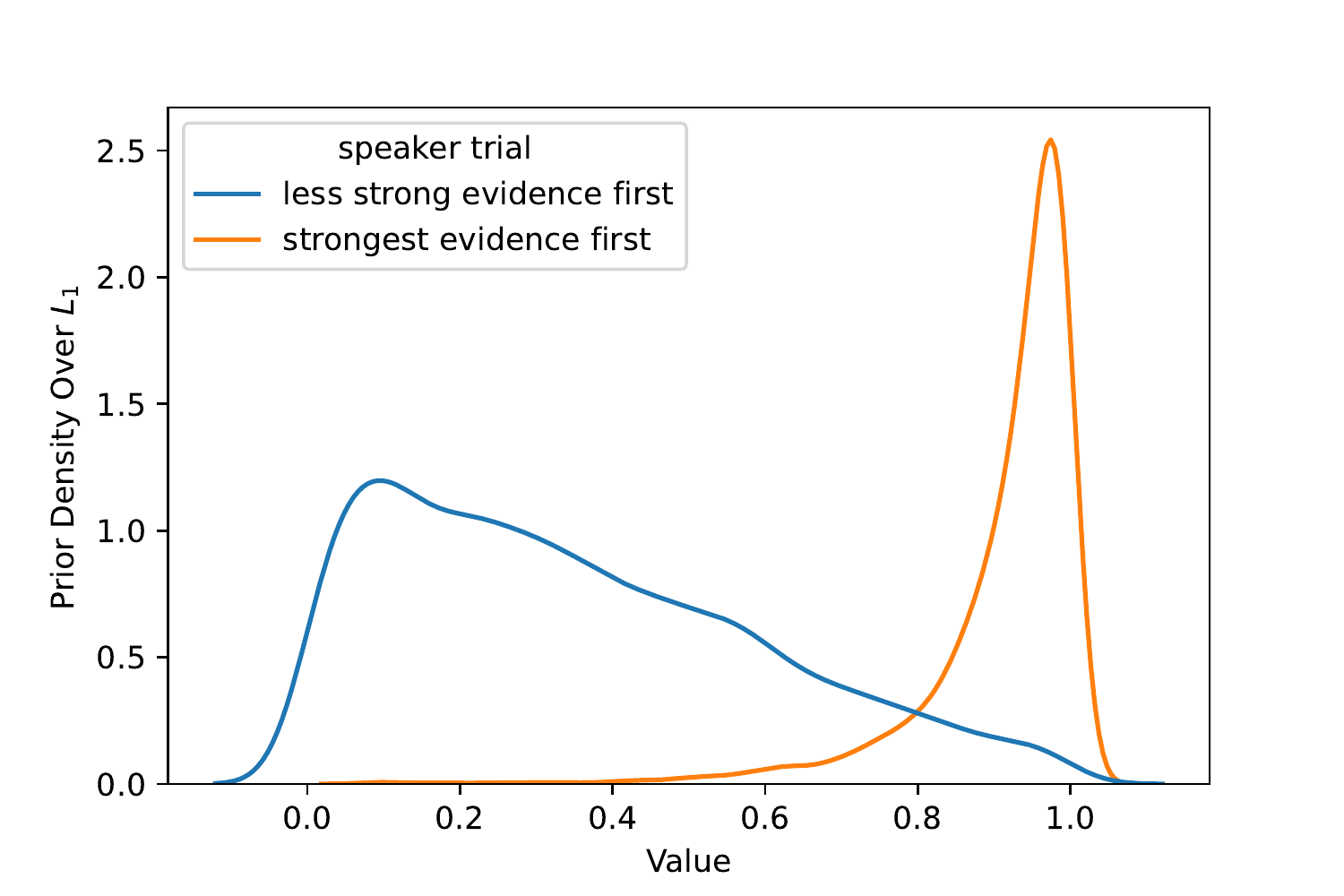}
        \caption{Full Bayesian posteriors for the parameters of the speaker-dependent RSA model. In the top panel, the MAP parameter values are found to be~${(\beta^*, o^*) = (2.26, -0.11)}$. The bottom panel shows the posteriors over mixture weights $p_z$ for the different speaker groups. The MAP parameter values for are $p_z=0.10$ for the less strong evidence group and $p_z=0.97$ for the strongest evidence group.}
        \label{fig:bayesian_posteriors}
\end{figure}

\begin{table}[b]
\begin{center}
    \captionsetup{font={small,stretch=.75}, strut=on}
    \begin{tabular}{p{0.1\linewidth}|p{0.9\linewidth}|}
    \hline
     \textbf{group} &  \textbf{What was your strategy for selecting sticks as a speaker?}  \\
         \hline\hline 
     \multirow{12}{0.1pt}{strongest evidence} & \texttt{If I need them to believe more than 5 inch i'd choose the biggest and opposite for below 5 inch} --- \emph{either choose the longest if I am blue, or the shortest if I am red} --- \texttt{I picked the longest or shortest stick based on what I wanted the judge to believe} --- \emph{Trying to show the extremes for each argument so the judge thinks the average is more likely to be closer to those} --- \texttt{I think it's best to show the longest/shortest stick you own - to make it appear that they're all very long/short} --- \emph{i guess it was to show extremes of the sizes of sticks i had, show the smallest on or the tallest one} --- \texttt{my strategy was to create the illusion  that the average lenght is bigger in the case I am the blue contestant by showing the longest sticks only, and the same with the red one only showing the shortest.} --- \emph{Pick the shortest or longest one to bump up or reduce the avearge}\\
         \hline
\multirow{6}{0.1pt}{weaker evidence} & \texttt{Selected slightly towards where the first stick suggested} --- \emph{I actually want to avoid the highest or lowest if I can at first to give the impression that you yourself have picked a more "average" stick.} ---  \texttt{Not going too far either way, but just enough to seem less obvious.} --- \emph{To show a slightly longer or shorter length than the average to try persuade the judge otherwise.} --- \texttt{show some variation to gain trust} --- \emph{try to keep them guessing} \\
    \end{tabular}
        \caption{Participants were presented with a free-response text field to explain their reasoning at the end of both phases.  Here we provide sample responses from the end of the \emph{speaker} phase, from both participants who expected the \emph{strongest} evidence and those who expected less strong evidence.}
        \label{tab:speaker_strategy}
\end{center}
\end{table}

\begin{table}
\begin{center}
    \captionsetup{font={small,stretch=.75}, strut=on}

    \begin{tabular}{p{0.1\linewidth}|p{0.9\linewidth}|}
    \hline
    group  &   \textbf{How did you reach your decision as a judge?}  \\
     \hline\hline 
   \multirow{12}{0.1pt}{strongest evidence} &    \texttt{6 is not very much over the average that their trying to prove - which makes me think that al the other sticks are even shorter than that."} 
    --- \emph{I was thinking that the pink player would choose the shortest stick, whilst the blue would choose the longest} ---
         \texttt{if 4cm was the shortest stick available then the maximum number of sticks above 5cm would be 4} ---
         \emph{blue showed me a very long stick meaning there would have to be an opposite short stick to average it out. pink however did not show a very short stick suggesting there aren't any.} --- \texttt{the blue would have shown a longer one if it was there} --- \emph{I assume that blue would be likely to pick the longest possible stick as they have an incentive to make me think the average is above 5in; if they only present a 6in stick, it is likely that the average is under 5in.} \\
         \hline
         \multirow{5}{0.1pt}{weaker evidence} &  \texttt{tried to do a average} --- \emph{Felt the pink player was bluffing} --- \texttt{thge average of the 2 sticks was shorter than 5} --- \emph{Looking at the average of the values I'd been given so far} --- \texttt{seemed similar to how I played it so assumed there were more long ones to come like in my strategy} --- \emph{the contestant is likely to trick you}
    \end{tabular}
        \caption{Sample responses from the end of the \emph{judge} phase, from both participants who expected the \emph{strongest} evidence and those who expected less strong evidence.}
        \label{tab:judge_strategy}
    \end{center}
\end{table}

%\bibliography{refs.bib}

\end{document}